\documentclass[10pt,twocolumn,letterpaper,table]{article}

\usepackage{cvpr}

\usepackage{xcolor}
\definecolor{mygreen}{RGB}{24,165,103}

\usepackage[colorlinks]{hyperref}
\usepackage[hyperpageref]{backref}

\usepackage{times}
\usepackage{epsfig}
\usepackage{graphicx}
\usepackage{amsmath}
\usepackage{amssymb}
\usepackage{enumitem}

\cvprfinalcopy 

\def\cvprPaperID{3614} 

\ifcvprfinal\fi

\usepackage{multirow}
\usepackage{multicol}
\usepackage{booktabs}
\usepackage{pifont}
\usepackage{iftex}
\usepackage[utf8]{inputenc}
\usepackage[T1]{fontenc}

\usepackage{amsthm}
\usepackage{dsfont}
\usepackage{pifont}
\usepackage{nth}
\usepackage{booktabs}
\usepackage[font=small,bf,skip=0.8em]{caption}
\usepackage[draft]{fixme}
\usepackage{inconsolata}

\usepackage[group-separator={,}]{siunitx}
\newcommand{\cmark}{\ding{51}}%
\newcommand{\xmark}{\ding{55}}%

\usepackage{url}

\usepackage{theoremref}
\numberwithin{equation}{section}
\theoremstyle{plain}

\newtheorem{theorem}{Theorem}

\newtheorem{proposition}[theorem]{Proposition}

\theoremstyle{definition}

\def\R{{\mathbb R}}

\let\on=\operatorname

\newcommand{\ud}{\,\mathrm{d}}

\DeclareMathOperator{\rp}{rp}
\DeclareMathOperator{\RP}{RP}

\newcommand{\eqdef}{\ensuremath{\stackrel{\mbox{\upshape\tiny def.}}{=}}}

\def\eg{\emph{e.g.}}  
\def\ie{\emph{i.e.}}

\def\wrt{w.r.t.}

\makeatother

\makeatletter
\newcommand*\rel@kern[1]{\kern#1\dimexpr\macc@kerna}
\newcommand*\widebar[1]{%
  \begingroup
  \def\mathaccent##1##2{%
    \rel@kern{0.8}%
    \overline{\rel@kern{-0.8}\macc@nucleus\rel@kern{0.2}}%
    \rel@kern{-0.2}%
  }%
  \macc@depth\@ne
  \let\math@bgroup\@empty \let\math@egroup\macc@set@skewchar
  \mathsurround\z@ \frozen@everymath{\mathgroup\macc@group\relax}%
  \macc@set@skewchar\relax
  \let\mathaccentV\macc@nested@a
  \macc@nested@a\relax111{#1}%
  \endgroup
}
\makeatother

\title{Metric Learning for Image Registration}
\author{Marc Niethammer\\
UNC Chapel Hill\\
{\tt\small \href{mailto:mn@cs.unc.edu}{mn@cs.unc.edu}}
\and
Roland Kwitt\\
University of Salzburg\\
{\tt\small \href{mailto:roland.kwitt@gmail.com}{roland.kwitt@gmail.com}}
\and
Fran\c{c}ois-Xavier~Vialard\\
LIGM, UPEM\\
{\tt\small \href{mailto:francois-xavier.vialard@u-pem.fr}{francois-xavier.vialard@u-pem.fr}}}

\begin{document}

\setlength{\abovedisplayskip}{0.65\abovedisplayskip}
\setlength{\belowdisplayskip}{0.65\belowdisplayskip}

\maketitle

\begin{abstract}
Image registration is a key technique in medical image analysis to estimate deformations between image pairs. A good deformation model is important for high-quality estimates. However, most existing approaches use ad-hoc deformation models chosen for mathematical convenience rather than to capture observed data variation. Recent deep learning approaches learn deformation models directly from data. However, they provide limited control over the spatial regularity of transformations. Instead of learning the entire registration approach, we learn a spatially-adaptive regularizer \emph{within} a registration model. This allows controlling the desired level of regularity and preserving structural properties of a registration model. For example, diffeomorphic transformations can be attained. Our approach is a radical departure from existing deep learning approaches to image registration by {\it embedding} a deep learning model in an optimization-based registration algorithm to parameterize and data-adapt the registration model itself. Source code is publicly-available at \textit{\url{https://github.com/uncbiag/registration}}.
\end{abstract}

\vspace{-0.5cm}

\section{Introduction}
\label{sec:intro}

Image registration is important in medical image analysis tasks to capture subtle, local deformations. Consequently, transformation models~\cite{holden2008review}, which parameterize these deformations, have large numbers of degrees of freedom, ranging from B-spline models with many control points, to non-parametric approaches~\cite{modersitzki2004numerical} inspired by continuum mechanics. Due to the large number of parameters of such models, deformation fields are typically regularized by \emph{directly} penalizing local changes in displacement or, more \emph{indirectly}, in velocity field(s) parameterizing a deformation.

Proper regularization is important to obtain high-quality deformation estimates. Most existing work simply imposes the same spatial regularity \emph{everywhere} in an image. This is unrealistic. For example, consider registering brain images with different ventricle sizes, or chest images with a moving lung, but a stationary rib cage, where different deformation scales are present in different image regions. Parameterizing such deformations from first principles is difficult and may be impossible for between-subject registrations. Hence, it is desirable to {\it learn local regularity} from data. One could replace the registration model entirely and learn a parameterized regression function $f_\Theta$ from a large dataset. At inference time, this function then maps a moving image to a target image \cite{de2017end}. However, regularity of the resulting deformation does not arise naturally in such an approach and typically needs to be enforced after the fact.

\begin{figure}[t!]
  \centering
    \includegraphics[width=0.85\columnwidth]{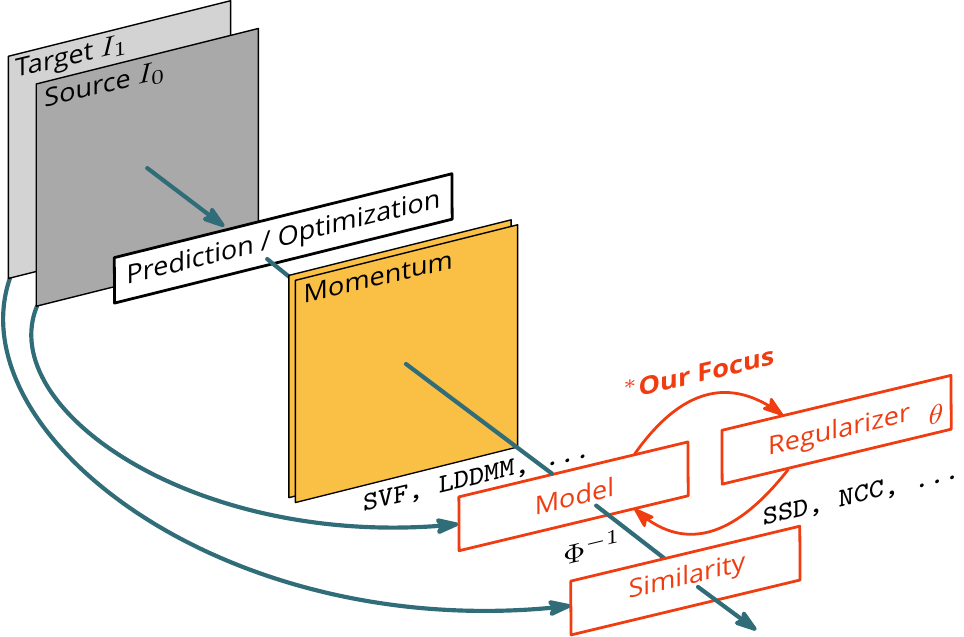}
              \caption{Architecture of our registration approach. We jointly optimize over the momentum, parameterizing the deformation $\Phi$, and the parameters, $\theta$, of a convolutional neural net (CNN). The CNN {\it locally} predicts multi-Gaussian kernel pre-weights which specify the regularizer. This approach constructs a metric such that diffeomorphic transformations can be assured in the continuum.} 
    \label{fig:overview}
    \vspace{-0.2cm}
\end{figure}

Existing non-parametric deformation models already yield good performance, are well understood, and use globally parameterized regularizers. Hence, we advocate building upon these models and to learn appropriate \emph{localized} parameterizations of the regularizer by leveraging large samples of training data. This strategy not only retains theoretical guarantees on deformation regularity, but also makes it possible to encode, in the metric, the intrinsic deformation model as supported by the data.


\noindent
    {\bf Contributions.} Our approach deviates from current approaches for (predictive) image registration in the following sense. Instead of replacing the entire registration model by a regression function, we retain the underlying registration model and \emph{learn} a spatially-varying regularizer. We build on top of a new \emph{vector momentum-parameterized stationary velocity field (vSVF)} registration model which allows us to guarantee that deformations are diffeomorphic even when using a learned regularizer. Our approach jointly optimizes the regularizer (parameterized by a deep network) and the registration parameters of the vSVF model. We show state-of-the art registration results and evidence for locally varying deformation models in real data.

\noindent
\textbf{Overview}.
Fig.~\ref{fig:overview} illustrates our key idea. We start with an initial momentum parameterization of a registration model, in particular, of the vSVF. Such a parameterization is important, because it allows control over deformation regularity \emph{on top of} the registration parameters. For a given source-target image-pair ($I_0$, $I_1$), we optimize over the momentum to obtain a spatial transformation $\Phi$ such that $I_0\circ\Phi^{-1}\approx I_1$. As the mapping from momentum to $\Phi$ is influenced by a regularizer expressing what transformations are desirable, we jointly optimize over the regularizer parameters, $\theta$, \emph{and} the momentum. Specifically, we use a spatially-adaptive regularizer, parameterized by a regression model (here, a CNN). Our approach naturally combines with a prediction model, \eg, \cite{yang2017quicksilver}, to obtain the momentum from a source-target image pair (avoiding optimization at runtime). Here, we \emph{numerically optimize} over the momentum for simplicity and leave momentum prediction to future work.

\noindent
\textbf{Organization}. In \S\ref{sec:background}, we review registration models, relations to our proposed approach and introduce the vSVF model. \S\ref{sec:metric_learning} describes our metric learning registration approach and \S\ref{sec:experiments} discusses experimental results. 
Finally, \S\ref{sec:discussion} summarizes the main points. \emph{Additional details can be found in the supplementary material.} 

\vspace{-0.25cm}

\section{Background on image registration}
\label{sec:background}

Image registration is typically formulated as an optimization problem of the form
\begin{equation}
  \gamma^* = \underset{\gamma}{\text{argmin}}~\lambda\ \text{Reg}[\Phi^{-1}(\gamma)] + \text{Sim}[I_0\circ\Phi^{-1}(\gamma),I_1].
\label{eqn:basicreg}
\end{equation}
Here, $\gamma$ parameterizes the deformation, $\Phi$, $\lambda\geq 0$, $\text{Reg}[\cdot]$ is a penalty encouraging spatially regular deformations and $\text{Sim}[\cdot,\cdot]$ penalizes dissimilarities between two images (\eg, sum-of-squared differences, cross-correlation or mutual information~\cite{hermosillo2002variational}). For low-dimensional parameterizations of $\Phi$, \eg, for affine or B-spline~\cite{rueckert1999nonrigid,modat2010fast} models, a regularizer may not be necessary. However, non-parametric registration models~\cite{modersitzki2004numerical} represent deformations via displacement, velocity, or momentum vector fields and require regularization for a well-posed optimization problem.

In medical image analysis, diffeomorphic transformations, $\Phi$, are often desirable to smoothly map between subjects or between subjects and an atlas space for local analyses. Diffeomorphisms can be guaranteed by estimating sufficiently smooth~\cite{dupuis1998} static or time-varying velocity fields, $v$. The transformation is then obtained via time integration, \ie, of $\Phi_t(x,t) = v\circ\Phi(x,t)$ (subscript $_t$ indicates a time derivative). Examples of such methods are the static velocity field (SVF)~\cite{vercauteren2009diffeomorphic} and the large displacement diffeomorphic metric mapping (LDDMM) registration models~\cite{beg2005,vialard2012diffeomorphic,hart2009optimal,avants2009advanced}. 

Non-parametric registration models require optimization over high-dimensional vector fields, often with millions of unknowns in 3D. Hence, numerical optimization can be slow. Recently, several approaches which learn a regression model to predict registration parameters from large sets of image pairs have emerged. Initial models based on deep learning~\cite{dosovitskiy2015flownet,ilg2017flownet} were proposed to speed-up optical flow computations~\cite{horn1981determining,beauchemin1995computation,brox2004high,borzi2003optimal,zach2007duality, sun2010secrets}. Non-deep-learning approaches for the regression of registration parameters have also been studied~\cite{wang2013joint,wang2015predict,chou20132d,cao2015semi,gutierrez2017guiding}. These approaches typically have no guarantees on spatial regularity or may not straightforwardly extend to 3D image volumes due to memory constraints. Alternative approaches have been proposed which can register 3D images~\cite{rohe2017svf,sokooti2017nonrigid,de2017end,hu2018label,balakrishnan2018unsupervised,fan2018adversarial} and assure diffeomorphisms~\cite{yang2016fast,yang2017quicksilver}. In these approaches, costly numerical optimization is only required during training of the regression model. Both end-to-end approaches~\cite{de2017end,hu2018label,balakrishnan2018unsupervised,fan2018adversarial} and approaches requiring the desired registration parameters during training exist~\cite{yang2016fast,yang2017quicksilver,rohe2017svf}. As end-to-end approaches differentiate through the transformation map, $\Phi$, they were motivated by the spatial transformer work~\cite{jaderberg2015spatial}.

One of the main conceptual downsides of current regression approaches is that they either explicitly encode regularity when computing the registration parameters to obtain the training data~\cite{yang2016fast,yang2017quicksilver,rohe2017svf}, impose regularity as part of the loss~\cite{hu2018label,balakrishnan2018unsupervised,fan2018adversarial} to avoid ill-posedness, or use low-dimensional parameterizations to assure regularity~\cite{sokooti2017nonrigid,de2017end}. Consequentially, these models \emph{do not} estimate a deformation model from data, but instead impose it by choosing a regularizer. Ideally, one would like a registration model which (1) regularizes according to deformations present in data, (2) is fast to compute via regression and which (3) retains desirable theoretical properties of the registration model (\eg, guarantees diffeomorphisms) even when predicting registration parameters via regression.

Approaches which predict momentum fields~\cite{yang2016fast,yang2017quicksilver} are fast and can guarantee diffeomorphisms. Yet, no model exists which estimates a local spatial regularizer of a form that guarantees diffeomorphic transformations and that can be combined with a fast regression formulation. Our goal is to close this gap via a momentum-based registration variant. While we will not explore regressing the momentum parameterization here, such a formulation is expected to be straightforward, as our proposed model has a momentum-parameterization similar to what has already been used successfully for regression with a deep network~\cite{yang2017quicksilver}.

\subsection{Fluid-type registration algorithms}
\label{subsection:fluid_registration}

To capture large deformations and to guarantee diffeomorphic transformations, registration methods inspired by fluid mechanics have been highly successful, \eg, in neuroimaging~\cite{avants2009advanced}. Our model follows this approach. The map $\Phi$ is obtained via time-integration of a sought-for velocity field $v(x,t)$. Specifically, $\Phi_t(x,t) = v(\Phi(x,t),t),~\Phi(x,0)=x$. For sufficiently smooth (\ie, sufficiently regularized) velocity fields, $v$, one obtains diffeomorphisms~\cite{dupuis1998}. The corresponding instance of Eq.~\eqref{eqn:basicreg} is
\begin{align*}
  v^* = &~\underset{v}{\text{argmin}}~\lambda \int_0^1 \|v\|_L^2~\mathrm{d}t + \text{Sim}[I_0\circ\Phi^{-1}(1),I_1],~\text{s.t.}\\
  & \Phi^{-1}_t + D\Phi^{-1}v=0,~\text{and}~\Phi^{-1}(0)=\text{id}\enspace.
\end{align*}
Here, $D$ denotes the Jacobian (of $\Phi^{-1}$), $\|v\|^2_L=\langle L^\dagger L v,v\rangle$ is a spatial norm defined using the differential operator $L$ and its adjoint $L^\dagger$. A specific $L$ implies an expected deformation model. In its simplest form, $L$ is \emph{spatially-invariant} and encodes a desired level of smoothness. As the vector-valued momentum, $m$, is given by $m=L^\dagger L v$, one can write the norm as $\|v\|_L^2 = \langle m,v\rangle$. 

In LDDMM~\cite{beg2005}, one seeks time-dependent vector fields $v(x,t)$. A simpler, but less expressive, approach is to use \emph{stationary velocity fields} (SVF), $v(x)$, instead~\cite{rohe2017svf}. While SVF's are  optimized directly over the velocity field $v$, we propose a \emph{vector momentum SVF (vSVF)} formulation, \ie, 
\begin{equation}
\begin{split}
    m^* = ~\underset{m_0}{\text{argmin}}~\lambda\langle m_0,v_0\rangle + \text{Sim}[I_0\circ\Phi^{-1}(1),I_1]\\
  ~\text{s.t.}~\Phi^{-1}_t + D\Phi^{-1}v=0\\
  ~\Phi^{-1}(0)=\text{id},~\text{and}~v_0=(L^\dagger L)^{-1}m_0\enspace,
  \label{eq:vsvf}
\end{split}
\end{equation}
which is optimized over the vector momentum $m_0$. vSVF is a simplification of vector momentum LDDMM~\cite{vialard2012diffeomorphic}. We use vSVF for simplicity, but our approach directly translates to LDDMM and is motivated by the desire for LDDMM regularizers adapting to a deforming image.

\section{Metric learning}
\label{sec:metric_learning}

In practice, $L$ is predominantly chosen to be spatially-invariant. Only limited work on \emph{spatially-varying} regularizers exists~\cite{risser2013piecewise,pace2013locally,stefanescu2004grid} and even less work focuses on \emph{estimating} a spatially-varying regularizer. A notable exception is the estimation of a spatially-varying regularizer in atlas-space~\cite{vialard2014spatially} which builds on a left-invariant variant of LDDMM~\cite{schmah2013left}. Instead, our goal is to \emph{learn} a spatially-varying regularizer which takes as inputs a momentum vector field and an image and computes a smoothed vector field. Therefore, our approach, not only leads to spatially varying metrics but can address pairwise registration, contrary to atlas-based learning methods, and it can adapt to deforming images during time integration for LDDMM\footnote{We use vSVF here and leave LDDMM as future work.}. We focus on extensions to the multi-Gaussian regularizer~\cite{risser2011simultaneous} as a first step, but note that learning more general regularization models would be possible.

\subsection{Parameterization of the metrics}
Metrics on vector fields of dimension $M$ are positive semi-definite (PSD) matrices of $M^2$ coefficients. Directly learning these $M^2$ coefficients is impractical, since for typical 3D image volumes $M$ is in the range of millions. We therefore restrict ourselves to a class of spatially-varying mixtures of Gaussian kernels.

\noindent
\textbf{Multi-Gaussian kernels.}
It is customary to directly specify 
the map from momentum to vector field via Gaussian smoothing, \ie, $v=G\star m$ (here, $\star$ denotes convolution). In practice, multi-Gaussian kernels are desirable~\cite{risser2011simultaneous} to capture multi-scale aspects of a deformation, where
\begin{equation}
  v=\left(\sum_{i=0}^{N-1} w_i G_i\right)\star m\enspace,~w_i\geq 0,~\sum_{i=0}^{N-1}w_i=1\enspace.
\label{eqn:mgkernel}
\end{equation}
$G_i$ is a normalized Gaussian centered at zero with standard deviation $\sigma_i$ and $w_i$ is a positive weight. The class of kernels that can be approximated by such a sum is already large\footnote{All the functions $h: \R_{>0} \mapsto \R$ such that $h(|x-y|)$ is a kernel on $\R^d$ for every $d \geq 1$ are in this class.}. 
A na\"ive approach to estimate the regularizer would be to learn $w_i$ and $\sigma_i$. However, estimating either the variances or weights benefits from adding penalty terms to encourage desired solutions. Assume, for simplicity, that we have a single Gaussian, $G$, $v=G\star m$, with standard deviation $\sigma$. As the Fourier transform is an $L^2$ isometry, we can write
\begin{multline}
  \int m(x)^\top v(x)~\mathrm{d}x = \langle m,v\rangle = \langle \hat{m},\hat{v}\rangle \\ = \langle \hat{v}/\hat{G},\hat{v}\rangle = \int e^{\pi^22\sigma^2 k^\top k}v(k)^
  \top v(k)~\mathrm{d}k\enspace,
\end{multline}
where $\hat{\cdot}$ denotes the Fourier transform and $k$ the frequency. Since $\hat{G}$ is a Gaussian without normalization constant, it follows that we need to explicitly penalize small $\sigma$'s if we want to favor smoother transformations (with large $\sigma$'s). 
Indeed, the previous formula shows that a constant velocity field has the same norm for every positive $\sigma$. More generally, in theory, it is possible to reproduce a given deformation by the use of different kernels. Therefore, a penalty function on the parameterizations of the kernel is desirable. We design this penalty via a simple form of \emph{optimal mass transport (OMT)} between the weights, as explained in the following.

\noindent
\textbf{OMT on multi-Gaussian kernel weights.}
Consider a multi-Gaussian kernel as in Eq.~\eqref{eqn:mgkernel}, with standard deviations $0<\sigma_0\leq \sigma_1 \leq \cdots \leq \sigma_{N-1}$. It would be desirable to obtain \emph{simple} transformations explaining deformations with large standard deviations. Interpreting the multi-Gaussian kernel weights as a distribution, the most desirable configuration would be $w_{i\neq N-1}=0,~w_{N-1}=1$, \ie, using only the Gaussian with largest variance. We want to penalize weight distributions deviating from this configuration, with the largest distance given to $w_0=1,~w_{i\neq 0}=0$. This can be achieved via an \emph{OMT penalty}. Specifically, we define this penalty on $w=[w_0,\ldots,w_{N-1}]$ as
\begin{equation}
  \text{OMT}(w) = \sum_{i=0}^{N-1}w_i\left|\log\frac{\sigma_{N-1}}{\sigma_i}\right|^r ,
\label{eqn:omtw}
\end{equation}
where $r\geq 1$ is a chosen power. In the following, we set $r=1$. This penalty is zero if $w_{N-1}=1$ and will have its largest value for $w_0=1$. It can be standardized as
\begin{equation}
  \widehat{\text{OMT}}(w) = \left|\log\frac{\sigma_{N-1}}{\sigma_0}\right|^{-r}\sum_{i=0}^{N-1}w_i\ \left|\log\frac{\sigma_{N-1}}{\sigma_i}\right|^r
\end{equation}
with $\widehat{\text{OMT}}(w)\in[0,1]$ by construction. 
\let\on=\operatorname

\vskip1ex
\noindent
\textbf{Localized smoothing.}
This multi-Gaussian approach is a \emph{global} regularization strategy, \ie, the \emph{same} multi-Gaussian kernel is applied \emph{everywhere}. This leads to efficient computations, but does not allow capturing localized changes in the deformation model. We therefore introduce {\it localized} multi-Gaussian kernels, embodying the idea of tissue-dependent localized regularization. Starting from a sum of kernels $\sum_{i = 0}^{N-1} w_i G_i$, we let the weights $w_i$ vary spatially, \ie, $w_i(x)$. To ensure diffeomorphic deformations, we set the weights $w_i(x) = G_{\sigma_{\text{small}}} \star \omega_i(x)$, where $\omega_i(x)$ are \emph{pre-weights} which are convolved with a Gaussian with small standard deviation. 
 An appropriate definition for how to use these weights to go from the momentum to the velocity is required to assure diffeomorphic transformations. Multiple approaches are possible. We use the model
\begin{equation}
\begin{split}
  v_0(x) & \eqdef ( K(w) \star m_0)(x)\\
  & = \sum_{i = 0}^{N-1} \sqrt{w_i(x)} \int_{y} G_i(| x - y |) \sqrt{w_i(y)} m_0(y) \on{d}\!y\,,\label{eq:sqrt_model}
\end{split}
\end{equation}
which, for spatially constant $w_i(x)$, reduces to the standard multi-Gaussian approach. 
In fact, this model guarantees diffeomorphisms, as long as the pre-weights are not too degenerate, as ensured by our model described hereafter. This fact is proven in the supplementary material (\ref{sec:sqrt_model}).
Motivated by the physical interpretation of these pre-weights and by diffeomorphic registration guarantees, we require a spatial regularization of these pre-weights via TV or $H^1$. We use color-TV \cite{blomgren1998color} for our experiments. As the spatial transformation is directly governed by the weights, we impose the OMT penalty locally. Based on Eq.~\eqref{eq:vsvf}, we optimize the following:
\begin{equation}
\begin{split}
  m^* = \underset{m_0}{\text{argmin}}~\lambda\langle m_0,v_0\rangle~+ \text{Sim}[I_0\circ\Phi^{-1}(1),I_1]~+\\
  ~\lambda_{\text{OMT}}\int \widehat{\text{OMT}}(w(x))~\mathrm{dx}~+\\ \lambda_{\text{TV}}\sqrt{\sum_{i=0}^{N-1}\left(\int \gamma(\|\nabla I_0(x)\|)\|\nabla \omega_i(x)\|_2~\mathrm{dx}\right)^2}\enspace,
\label{eqn:vsvf}
\end{split}
\end{equation}
subject to the constraints $\Phi^{-1}_t + D\Phi^{-1}v=0$ and $\Phi^{-1}(0)=\text{id}$; $\lambda_{\text{TV}},\lambda_{\text{OMT}}\geq 0$. 
The partition of unity defining the metric, intervenes in the $L^2$ scalar product $\langle m_0,v_0 \rangle$. 

Further, in Eq.~\eqref{eqn:vsvf}, the OMT penalty is integrated point-wise over the image-domain to support spatially-varying weights; $\gamma(x)\in\mathbb{R}^+$ is an 
\emph{edge indicator function}, \ie, 
$$\gamma(\|\nabla I\|)=(1+\alpha\|\nabla I\|)^{-1},~\text{with}~\alpha>0\enspace,$$ 
to encourage weight changes coinciding with image edges. 

\noindent
\textbf{Local regressor.}
To learn the regularizer, we propose a {\it local regressor} from the image and the momentum to the pre-weights of the multi-Gaussian. Given the momentum $m$ and image $I$ (the source image $I_0$ for vSVF; $I(t)$ at time $t$ for LDDMM) we learn a mapping of the form: 
$f_{\theta}:\mathbb{R}^d\times\mathbb{R}\to\Delta^{N-1}$ 
, where $\Delta^{N-1}$ is the $N-1$ unit/probability simplex\footnote{We only explore mappings dependent on the source image $I_0$ in our experiments, but more general mappings also depending on the momentum, for example, should be explored in future work.}. We will parametrize $f_{\theta}$ by a CNN in  
\S\ref{subsection:cnn_regressor}. The following attractive properties are worth pointing out:
\begin{itemize}[leftmargin=14pt]
\setlength\itemsep{-1pt}
\item[1)] The variance of the multi-Gaussian is bounded by the variances of its components. We retain these bounds and can therefore \emph{specify a desired regularity level}.
\item[2)] A globally smooth set of velocity fields is still computed (in Fourier space) which allows capturing large-scale regularity without a large receptive field of the local regressor. Hence, the CNN can be kept  efficient.
  \item[3)] The local regression strategy makes the approach suitable for more general registration models, \eg, for LDDMM, where one would like the regularizer to follow the \emph{deforming} source image over time.
\end{itemize}

\subsubsection{Learning the CNN regressor}
\label{subsection:cnn_regressor}

For simplicity we use a fairly shallow CNN with two layers of 
filters and leaky ReLU (\texttt{lReLU}) \cite{Maas13a} activations.
In detail, the data flow is as follows: 
\texttt{conv$(d+1,n_1)$} $\rightarrow$ \texttt{BatchNorm} $\rightarrow$ \texttt{lReLU} $\rightarrow$ \texttt{conv$(n_1,N)$} $\rightarrow$ \texttt{BatchNorm} $\rightarrow$ \texttt{weighted-linear-softmax}. Here \texttt{conv$(a,b)$} denotes a convolution layer with $a$ input channels and $b$ output feature maps. We used $n_1=20$ for our experiments and convolutional filters of spatial size $5$ ($5\times 5$ in 2D and $5\times 5\times 5$ in 3D). The \texttt{weighted-linear-softmax} activation function, which we formulated, maps inputs to $\Delta^{N-1}$. We designed it such that it operates around a setpoint of weights $w_i$ which correspond to the global weights of the multi-Gaussian kernel. This is useful to allow models to start training from a pre-specified, reasonable initial configuration of global weights, parameterizing the regularizer. Specifically, we define the {\it weighted linear softmax} $\sigma_w: \mathbb{R}^k \to \Delta^{N-1}$ as \begin{equation}
  \sigma_w(z)_j = \frac{\text{clamp}_{0,1}(w_j+z_j-\overline{z})}{\sum_{i=0}^{N-1} \text{clamp}_{0,1}(w_i+z_i-\overline{z})} \enspace,
  \label{eq:weighted_linear_softmax}
\end{equation}
where $\sigma_w(z)_j$ denotes the $j$-th component of the output, $\overline{z}$ is the mean of the inputs, $z$, and the clamp function clamps the values to the interval $[0,1]$. 
The removal of the mean in Eq.~\eqref{eq:weighted_linear_softmax} assures that one moves along the probability simplex. That is, if one is outside the clamping range, then 
$$\sum_{i=0}^{N-1} \text{clamp}_{0,1}(w_i+z_i-\overline{z}) = \sum_{i=0}^{N-1} w_i + z_i-\overline{z} = \sum_{i=0}^{N-1} w_i = 1$$
and consequentially, in this range, $\sigma_w(z)_j=w_j+z_j-\overline{z}$. This is linear in $z$ and moves along the tangent plane of the probability simplex by construction. 
As a CNN with small initial weights will produce an output close to zero, the output of $\sigma_w(z)$ will initially be close to the desired setpoint weights, $w_j$, of the multi-Gaussian kernel. 
Once the pre-weights, $\omega_i(x)$, have been obtained via this CNN, we compute multi-Gaussian weights via Gaussian smoothing. We use $\sigma=0.02$ in 2D and $\sigma=0.05$ in 3D throughout all experiments (\S\ref{sec:experiments}).

\subsection{Discretization, optimization, and training}
\label{subsec:discretization_optimization_training}

\noindent
{\bf Discretization.} We discretize the registration model using central differences for spatial derivatives and 20 steps in 2D (10 in 3D) of \nth{4} order Runge-Kutta integration in time. Gaussian smoothing is done in the Fourier domain. The entire model is implemented in \texttt{PyTorch}\footnote{Available at \url{https://github.com/uncbiag/registration}, also including various other registration models such as LDDMM.}; all gradients are computed by automatic differentiation \cite{Paszke17a}.

\noindent
    {\bf Optimization.} Joint optimization over the momenta of a set of registration pairs and the network parameters is difficult in 3D due to GPU memory limitations. Hence, we use a customized variant of stochastic gradient descent (SGD) with Nesterov momentum ($0.9$) \cite{Sutskever13a}, where we split optimization variables (1) that are {\it shared} and (2) {\it individual} between registration-pairs. Shared parameters are for the CNN. Individual parameters are the momenta. Shared parameters are kept in memory and individual parameters, including their current optimizer states, are saved and restored for every random batch. We use a batch-size of $2$ in 3D and $100$ in 2D and perform 5 SGD steps for each batch. 
    Learning rates are $1.0$ and $0.25$ for the individual and the shared parameters in 3D and $0.1$ and $0.025$ in 2D, respectively. We use gradient clipping (at a norm of one, separately for the gradients of the shared and the individual parameters) to help balance the energy terms. We use \texttt{PyTorch}'s {\tt ReduceLROnPlateau} learning rate scheduler with a reduction factor of 0.5 and a patience of 10 to adapt the learning rate during training. 

    \noindent {\bf Curriculum strategy:} Optimizing \emph{jointly} over momenta, global multi-Gaussian weights and the CNN does not work well in practice. Instead, we train in two stages: (1) In the initial global stage, we pick a reasonable set of global Gaussian weights and optimize only over the momenta. This allows further optimization from a reasonable starting point. Local adaptations (via the CNN) can then immediately capture local effects rather than initially being influenced by large misregistrations. In all  experiments, we chose these global weights to be linear with respect to their associated variances, \ie, $w_i = \sigma_i^2/(\sum_{j=0}^{N-1}\sigma_j^2)$. Then, (2) starting from the result of (1), we optimize over the momenta \emph{and} the parameters of the CNN to obtain spatially-localized weights. We refer to stages (1) and (2) as \emph{global} and \emph{local} optimization, respectively. 
      In 2D, we run global/local optimization for 50/100 epochs. In 3D, we run for 25/50 epochs. Gaussian variances are set to $\{0.01,0.05,0.1,0.2\}$ for images in $[0,1]^d$. We use normalized cross correlation (NCC) with $\sigma=0.1$ as similarity measure. See \S\ref{sec:implementation_details} of the supplementary material for further implementation details.
\vspace{-0.1cm}

\section{Experiments}
\label{sec:experiments}

We tested our approach on three dataset types: (1) 2D synthetic data with known ground truth (\S\ref{subsec:synthetic_experiment}), (2) 2D slices of a real 3D brain magnetic resonance (MR) images (\S\ref{subsec:real_2d_experiment}), and (3) multiple 3D datasets of brain MRIs (\S\ref{subsec:real_3d_experiment}). Images are first affinely aligned and intensity standardized by matching their intensity quantile functions to the average quantile function over all datasets. We compute deformations at half the spatial resolution in 2D  ($0.4$ times in 3D) and upsample $\Phi^{-1}$ to the original resolution when evaluating the similarity measure so that fine image details can be considered. This is not necessary in 2D, but essential in 3D to reduce GPU memory requirements. We use this approach in 2D for consistency. 

All evaluations (except for 
\S\ref{subsec:real_2d_experiment} and for the within dataset results of \S\ref{subsec:real_3d_experiment}) are with respect to a separate testing set. For testing, the previously learned regularizer parameters are fixed and numerical optimization is over momenta only (in particular, 250/500 iterations in 2D and 150/300 in 3D for global/local optimization).
  
\subsection{Results on 2D synthetic data}
\label{subsec:synthetic_experiment}

We created 300 synthetic  $128 \times 128$ image pairs of randomly deformed concentric rings (see supplementary material, \S\ref{sec:synthetic_experiment_setup}). Shown results are on 100 separate test cases.

Fig.~\ref{fig:synth_example_results_images} shows registrations for $\lambda_{\text{OMT}}\in\{15,50,100\}$. The TV penalty was set to $\lambda_{\text{TV}}=0.1$. 
The estimated standard deviations, $\sigma^2(x)=\sum_{i=0}^{N-1}w_i(x)\sigma_i^2$, capture the trend of the ground truth, showing a large standard deviation (\ie, high regularity) in the background and the center of the image and a smaller standard deviation in the outer ring. The standard deviations are stable across OMT penalties, but show slight increases with higher OMT values. Similarly, deformations get progressively more regular with larger OMT penalties (as they are regularized more strongly), but visually all registration results show very similar good correspondence.
Note that while TV was used to train the model, the CNN output is not explicitly TV regularized, but nevertheless is able to produce largely constant regions that are well aligned with the boundaries of the source image. Fig.~\ref{fig:synth_example_results_weights} shows the corresponding estimated weights. They are stable for a wide range of OMT penalties.

\begin{figure}[t!]
\centering{
\includegraphics[width=\columnwidth]{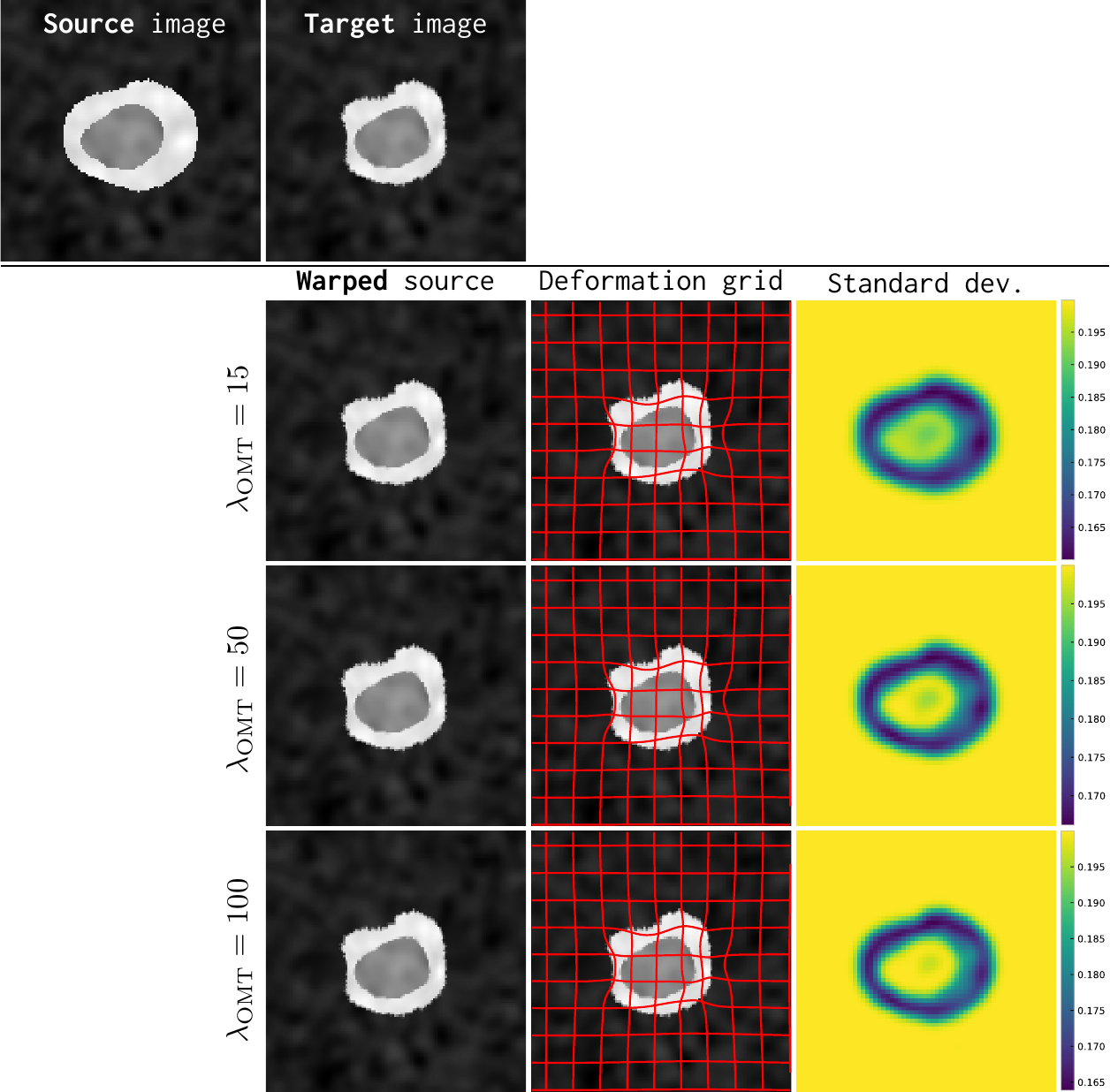}}
  \caption{Example registration results using local metric optimization for the synthetic test data. Results are shown for different values of $\lambda_{\text{OMT}}$ with the total variation penalty fixed to $\lambda_{\text{TV}}=0.1$. Visual correspondence between the warped source and the target images are high for all settings. Estimates for the standard deviation stay largely stable. However, deformations are slightly more regularized for higher OMT penalties. This can also be seen based on the standard deviations (\emph{best viewed zoomed}).}
  \label{fig:synth_example_results_images}
\end{figure}

\begin{figure}
\centering{
\includegraphics[width=\columnwidth]{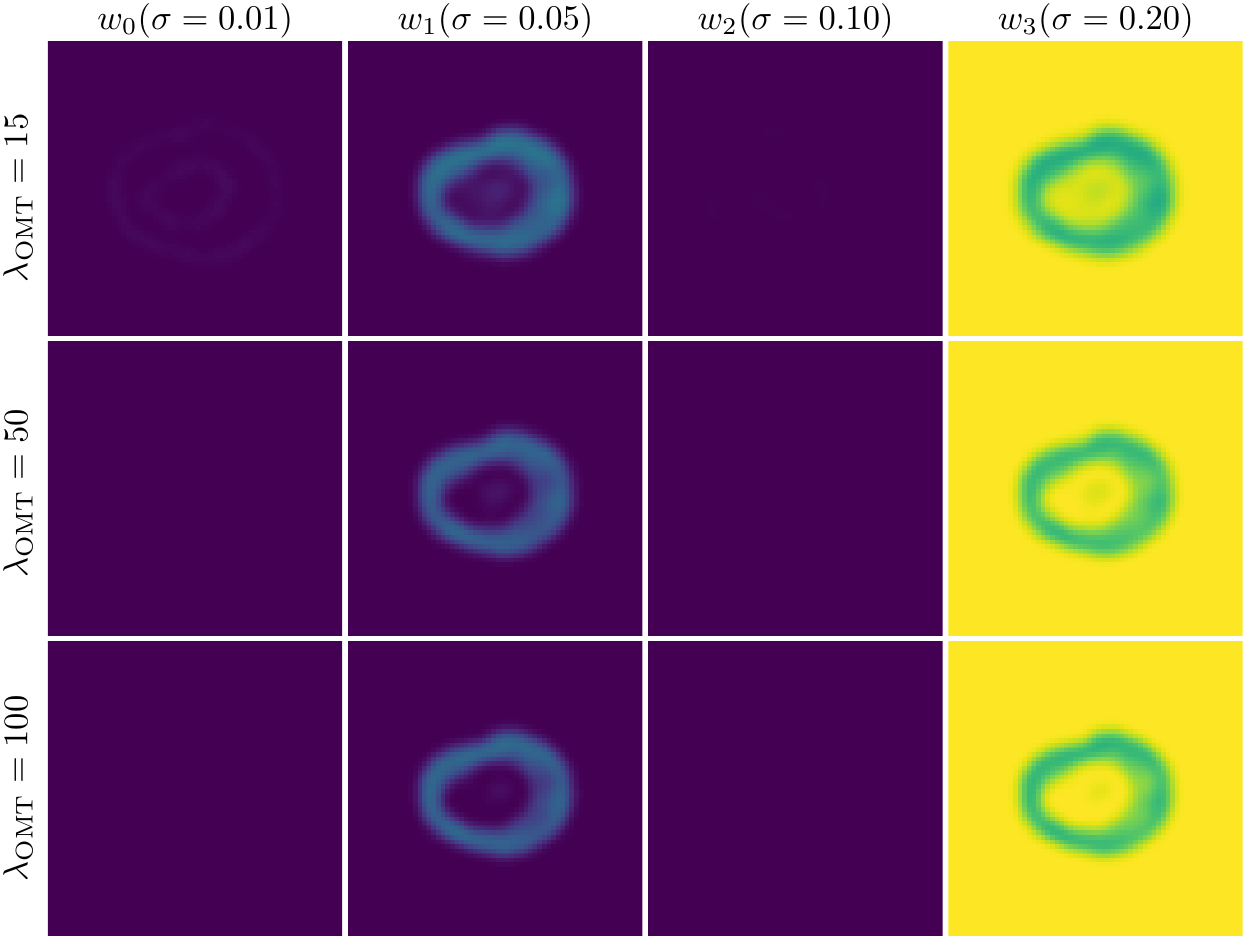}}
\caption{Estimated multi-Gaussian weights (blue=0; yellow=1) for the registrations in Fig.~\ref{fig:synth_example_results_images} \wrt~different $\lambda_{\text{OMT}}$'s.
Weight estimates are very stable across $\lambda_{\text{OMT}}$. While the overall standard deviation (Fig.~\ref{fig:synth_example_results_images}) approximates the ground truth, the weights for the outer ring differ (ground truth weights are $[0.05, 0.55, 0.3, 0.1]$) from the ground truth. They approximately match for the background and the interior (ground truth $[0,0,0,1]$).}
  \label{fig:synth_example_results_weights}
\end{figure}

\begin{figure}
\centering{
\includegraphics[width=\columnwidth]{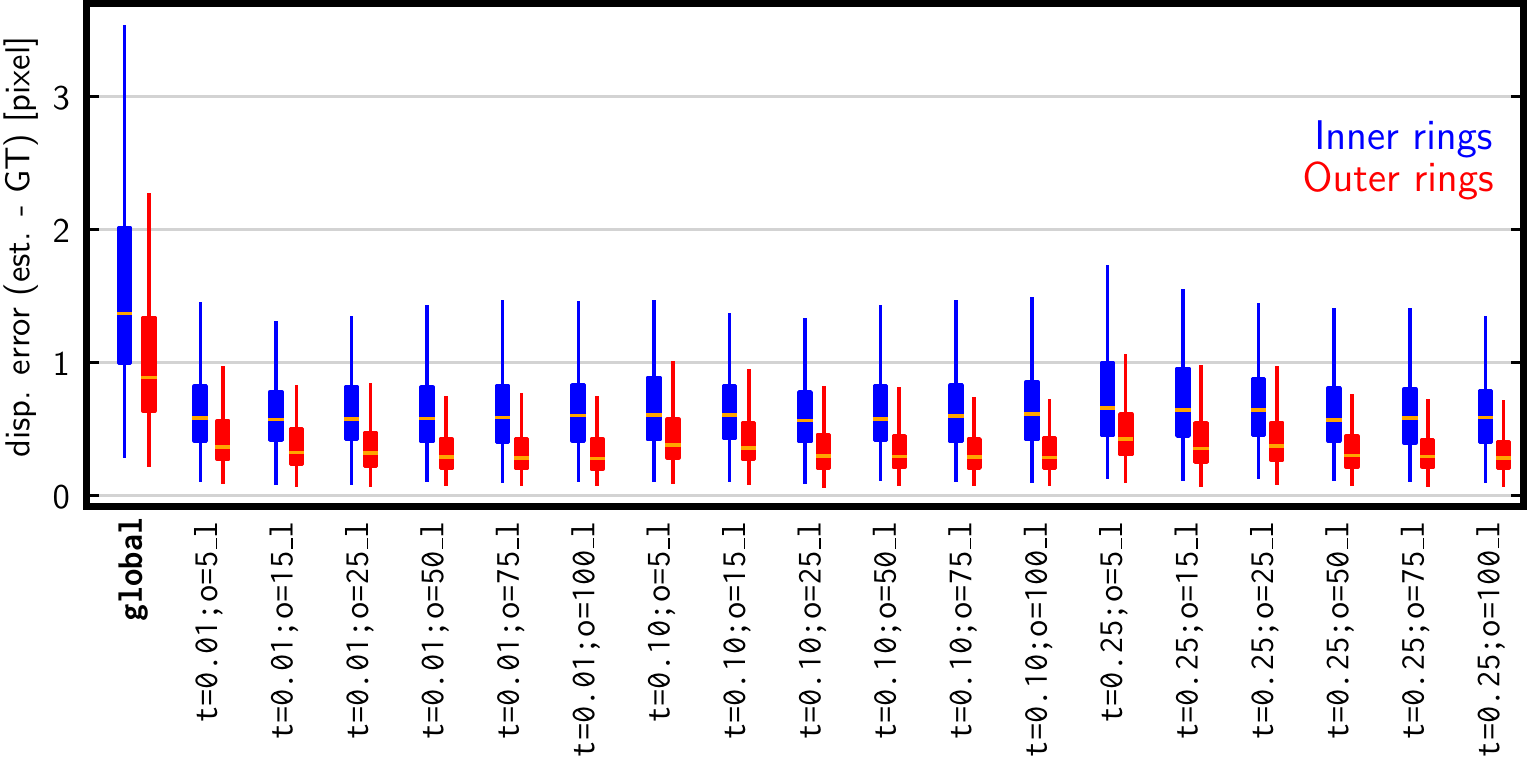}}
  \caption{\emph{Displacement error} (in pixel) with respect to the ground truth (GT) for various values of the total variation penalty, $\lambda_{\text{TV}}$ (\texttt{t}) and the OMT penalty, $\lambda_{\text{OMT}}$ (\texttt{o}). Results for the \textcolor{blue}{inner} and the \textcolor{red}{outer} rings show subpixel registration accuracy for all \emph{local} metric optimization results (\texttt{*\_l}). Overall, local metric optimization substantially improves registrations over the results obtained via initial global multi-Gaussian regularization (\texttt{global}).
    \label{fig:displacement_errors_within_shape}}
\end{figure}

Finally, Fig.~\ref{fig:displacement_errors_within_shape} shows displacement errors relative to the ground truth deformation for the interior and the exterior ring of the shapes. Local metric optimization significantly improves registration (over initial global multi-Gaussian regularization); these results are stable across a wide range of penalties with median displacement errors $<1$ pixel.

\subsection{Results on real 2D data}
\label{subsec:real_2d_experiment}

We used the same settings as for the synthetic dataset. However, here our results are for 300 random registration pairs of axial slices of the LPBA40 dataset~\cite{klein2009}. 

Fig.~\ref{fig:real_example_results_images} shows results for $\lambda_{\text{OMT}}\in\{15,50,100\}$; $\lambda_{\text{TV}}=0.1$. Larger OMT penalties yield larger standard deviations and consequentially more regular deformations. Most regions show large standard deviations (high regularity), but lower values around the ventricles and the brain boundary -- areas which may require substantial deformations. 

Fig.~\ref{fig:real_example_results_weights} shows the corresponding estimated weights. We have no ground truth here, but observe that the model produces consistent regularization patterns for all shown OMT values (\{15,50,100\}) and allocates almost all weights to the Gaussians with the lowest and the highest standard deviations, respectively. As $\lambda_{\text{OMT}}$ increases, more weight shifts from the smallest to the largest Gaussian.

\begin{figure}[t!]
\centering{
\includegraphics[width=\columnwidth]{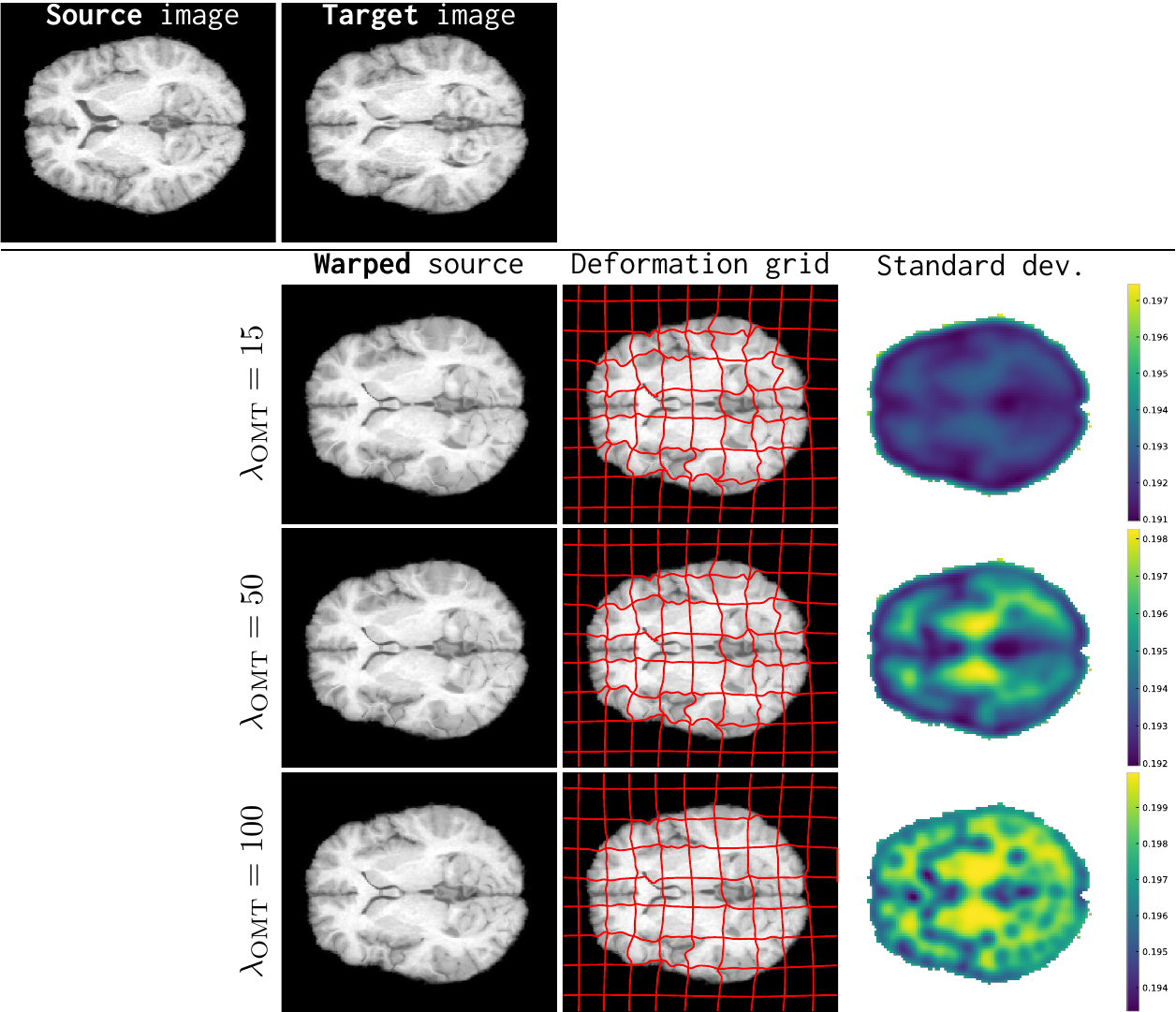}}
\caption{Example registration results using local metric optimization for different $\lambda_{\text{OMT}}$'s and $\lambda_{\text{TV}}=0.1$. Visual correspondences between the warped source images and the target image are high for all values of the OMT penalty. Standard deviation estimates capture the variability of the ventricles and increased regularity with increased values for $\lambda_{\text{OMT}}$ (\emph{best viewed zoomed}).}
  \label{fig:real_example_results_images}
  \vspace{-0.3cm}
\end{figure}

\begin{figure}[t!]
\centering{
\includegraphics[width=\columnwidth]{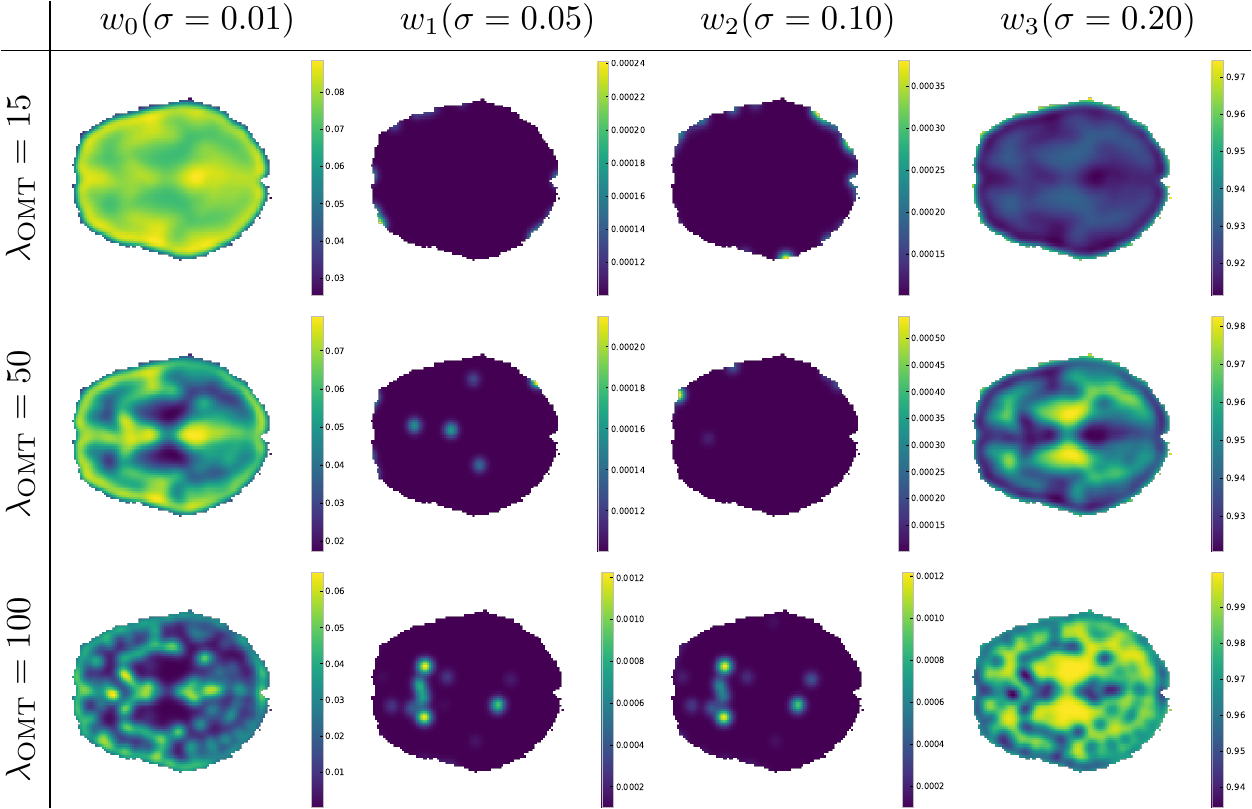}}
  \caption{Estimated multi-Gaussian weights for different $\lambda_{\text{OMT}}$ for real 2D data. Weights are mostly allocated to the Gaussian with the largest standard deviation (see colorbars; best viewed zoomed). A shift from $w_0$ to $w_3$ can be observed for larger values of $\lambda_{\text{OMT}}$. While weights shift between OMT setting, the ventricle area is always associated with more weight on $w_0$ (\emph{best viewed zoomed}).} 
  \label{fig:real_example_results_weights}
  \vspace{-0.5cm}
\end{figure}

\subsection{Results on real 3D data} 
\label{subsec:real_3d_experiment}

We experimented using the 3D CUMC12, MGH10, and IBSR18 datasets~\cite{klein2009}. These datasets contain 12, 10, and 18 images. \emph{Registration evaluations are with respect to all 132 registration pairs of CUMC12}. We use $\lambda_{\text{OMT}}=50$, $\lambda_{\text{TV}}=0.1$ for all tests\footnote{We did not tune these parameters and better settings may be possible.}. Once the regularizer has been learned, we keep it fixed and optimize for the vSVF vector momentum. We trained independent models on CUMC12, MGH10, and IBSR18 using 132 image pairs on CUMC12, 90 image pairs on MGH10, and a random set of 150 image pairs on IBSR18. We tested the resulting three models on CUMC12 to assess the performance of a dataset-specific model and the ability to transfer models across datasets.

Tab.~\ref{tab:target_overlap_3d_cumc12} and Fig.~\ref{fig:boxplot_overlap_3d_test_cumc12} compare to the registration methods in~\cite{klein2009} and across different stages of our approach for different training/testing pairs. We also list the 
performance of the most recent VoxelMorph (\texttt{VM}) variant \cite{Dalca18a}. We kept the original architecture configuration, swept over a selection of VoxelMorph's hyperparameters and report the best results here. Each VoxelMorph model was trained for 300 epochs which, in our
experiments, was sufficient for convergence. Overall, our approach shows the best results among all models when trained/tested on CUMC12 (\texttt{c/c local}); though results are not significantly better than for SyN, SPM5D, and VoxelMorph. Local metric optimization shows strong improvements over initial global multi-Gaussian regularization. Models trained on MGH10 and IBSR18 (\texttt{m/c local} and \texttt{i/c local}) also show good performance, slightly lower than for the model trained on CUMC12 itself, but higher than all other competing methods. This indicates that the trained models transfer well across datasets. While the top competitor in terms of median overlap (SPM5D) produces outliers (cf. Fig.~\ref{fig:boxplot_overlap_3d_test_cumc12}), our models do not. In case of VoxelMorph we observed that adding more training pairs (\ie, using all pairs of IBSR18, MGH18 \& LBPA40)  did not improve results (\emph{cf.} Tab.~\ref{tab:target_overlap_3d_cumc12} \texttt{*/c VM}). 

In Tab.~\ref{tab:jacobian_across_stages_cumc12_3d}, we list statistics for the determinant of the Jacobian of $\Phi^{-1}$ on CUMC12, where the model was also trained on. This illustrates how transformation regularity changes between the global and the local regularization approaches. As expected, the initial global multi-Gaussian regularization results in highly regular registrations (\ie, determinant of Jacobian close to one). Local metric optimization achieves significantly improved target volume overlap measures (Fig.~\ref{fig:boxplot_overlap_3d_test_cumc12}) while keeping good spatial regularity, clearly showing the utility of our local regularization model. Note that all reported determinant of Jacobian values in Tab.~\ref{tab:jacobian_across_stages_cumc12_3d} are positive, indicating no foldings, which is consistent with our diffeomorphic guarantees; though these are only guarantees for the continuous model at convergence, which do not consider potential discretization artifacts.

\renewcommand{\tabcolsep}{3pt}

\begin{table}
\begin{tiny}
\centering{
    \scalebox{1.09}{
		\begin{tabular}{| l | c | c | c | c | c | c | c | c | c | c |}
			\hline
			\textbf{Method} & \textbf{mean} & \textbf{std} & \textbf{1\%} & \textbf{5\%} & \textbf{50\%} & \textbf{95\%} & \textbf{99\%} & p & MW-stat & sig? \\
			\hline
			 \texttt{FLIRT} & 0.394 & 0.031 & 0.334 & 0.345 & 0.396 & 0.442 & 0.463 & \textless\num{1e-10} & 17394.0 & \cmark \\
			  \texttt{AIR} & 0.423 & 0.030 & 0.362 & 0.377 & 0.421 & 0.483 & 0.492 & \textless\num{1e-10} & 17091.0 & \cmark \\
			  \texttt{ANIMAL} & 0.426 & 0.037 & 0.328 & 0.367 & 0.425 & 0.483 & 0.498 & \textless\num{1e-10} & 16925.0 & \cmark \\
			  \texttt{ART} & 0.503 & 0.031 & 0.446 & 0.452 & 0.506 & 0.556 & 0.563 & \textless\num{1e-4} & 11177.0 & \cmark \\
			  \texttt{Demons} & 0.462 &  \cellcolor{green!30}{\bf 0.029} & 0.407 & 0.421 & 0.461 & 0.510 & 0.531 & \textless\num{1e-10} & 15518.0 & \cmark \\
			  \texttt{FNIRT} & 0.463 & 0.036 & 0.381 & 0.410 & 0.463 & 0.519 & 0.537 & \textless\num{1e-10} & 15149.0 & \cmark \\
			  \texttt{Fluid} & 0.462 & 0.031 & 0.401 & 0.410 & 0.462 & 0.516 & 0.532 & \textless\num{1e-10} & 15503.0 & \cmark \\
			  \texttt{SICLE} & 0.419 & 0.044 & 0.300 & 0.330 & 0.424 & 0.475 & 0.504 & \textless\num{1e-10} & 17022.0 & \cmark \\
			  \texttt{SyN} & 0.514 & 0.033 & 0.454 & 0.460 & 0.515 & 0.565 & 0.578 & 0.073 & 9677.0 & \xmark \\
			  \texttt{SPM5N8} & 0.365 & 0.045 & 0.257 & 0.293 & 0.370 & 0.426 & 0.455 & \textless\num{1e-10} & 17418.0 & \cmark \\
			  \texttt{SPM5N} & 0.420 & 0.031 & 0.361 & 0.376 & 0.418 & 0.471 & 0.494 & \textless\num{1e-10} & 17160.0 & \cmark \\
			  \texttt{SPM5U} & 0.438 & 0.029 & 0.373 & 0.394 & 0.437 & 0.489 & 0.502 & \textless\num{1e-10} & 16773.0 & \cmark \\
			  \texttt{SPM5D} & 0.512 & 0.056 & 0.262 & 0.445 & 0.523 & 0.570 & 0.579 & 0.311 & 9043.0 & \xmark \\ \hline\hline
			  \texttt{c/c VM} & 0.517 & 0.034 &  \cellcolor{green!30}{\bf 0.456} & 0.460 & 0.518 & 0.571 & 0.580 & 0.244 & 9211.0 & \xmark \\
			  \texttt{m/c VM} & 0.510 & 0.034 & 0.448 & 0.453 & 0.509 & 0.564 & 0.574 & 0.011 & 10197.0 & \cmark \\
			  \texttt{i/c VM} & 0.510 & 0.034 & 0.450 & 0.453 & 0.508 & 0.564 & 0.573 & 0.012 & 10170.0 & \cmark \\
			  \texttt{*/c VM} & 0.509 & 0.033 & 0.450 & 0.453 & 0.509 & 0.561 & 0.570 & 0.007 & 10318.0 & \cmark \\ \hline\hline
			  \texttt{m/c global} & 0.480 & 0.031 & 0.421 & 0.430 & 0.482 & 0.530 & 0.543 & \textless\num{1e-10} & 13864.0 & \cmark \\
			  \texttt{m/c local} & 0.517 & 0.034 & 0.454 & 0.461 & 0.521 & 0.568 & 0.578 & 0.257 & 9163.0 & \xmark \\ \hline\hline
			  \texttt{c/c global} & 0.480 & 0.031 & 0.421 & 0.430 & 0.482 & 0.530 & 0.543 & \textless\num{1e-10} & 13864.0 & \cmark \\
			  \texttt{c/c local} &  \cellcolor{green!30}{\bf 0.520} & 0.034 & 0.455 &  \cellcolor{green!30}{\bf 0.463} &  \cellcolor{green!30}{\bf 0.524} &  \cellcolor{green!30}{\bf 0.572} &  \cellcolor{green!30}{\bf 0.581} & - & - & - \\ \hline\hline
			  \texttt{i/c global} & 0.480 & 0.031 & 0.421 & 0.430 & 0.482 & 0.530 & 0.543 & \textless\num{1e-10} & 13863.0 & \cmark \\
			  \texttt{i/c local} & 0.518 & 0.035 & 0.454 & 0.460 & 0.522 & 0.571 & 0.581 & 0.338 & 8972.0 & \xmark \\
			\hline
		\end{tabular}
}}
  \caption{Statistics for mean (over all labeled brain structures, disregarding the background) target overlap ratios on CUMC12 for different methods. Prefixes for results based on global and local regularization indicate training/testing combinations identified by first initials of the datasets. For example, \texttt{m/c} means trained/tested on MGH10/CUMC12. Statistical results are for the null-hypothesis of equivalent mean target overlap with respect to \texttt{c/c local}. Rejection of the null-hypothesis (at $\alpha=0.05$) is indicated with a check-mark (\cmark). All $p$-values are computed using a paired one-sided Mann Whitney rank test~\cite{mann1947} and corrected for multiple comparisons using the Benjamini-Hochberg~\cite{benjamini1995} procedure with a family-wise error rate of $0.05$. Best results are \textbf{bold}, showing that our methods exhibits state-of-the-art performance.}
        \label{tab:target_overlap_3d_cumc12}
\end{tiny}
\end{table}

\begin{figure}
  \centering{
  \includegraphics[width=0.97\columnwidth]{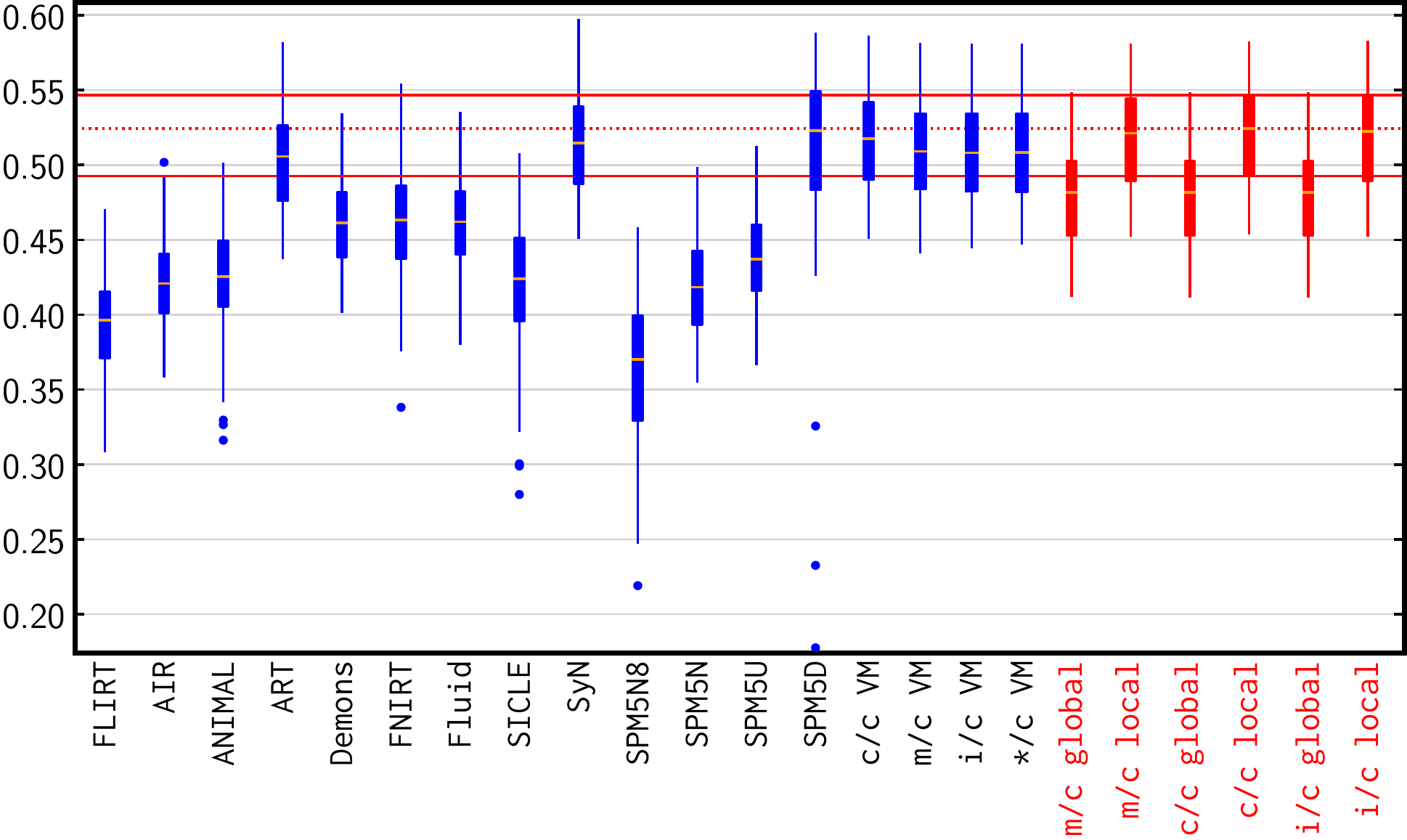}}
  
  \caption{Mean target overlap ratios on CUMC12 (in 3D) with $\lambda_{\text{TV}}=0.1$ and $\lambda_{\text{OMT}}=50$. Our approach (marked \textcolor{red}{red}) gives the best result overall. Local metric optimization greatly improves results over the initial global multi-Gaussian regularization. Best results are achieved for the model that was trained on this dataset (\texttt{c/c local}), but models trained on MGH10 (\texttt{m/c local}) and on IBSR18 (\texttt{i/c local}) transfer well and show almost the same level of performance.
    The dashed line is the median mean target overlap ratio (\ie, mean over all labels, median over all registration pairs).}
  \label{fig:boxplot_overlap_3d_test_cumc12}
  \vspace{-0.3cm}
\end{figure}

\begin{table}
  \centering{
  \begin{scriptsize}
    \begin{tabular}{|l|c|c|c|c|c|c|}
\hline
     & \textbf{mean} & \textbf{1\%} & \textbf{5\%} & \textbf{50\%} & \textbf{95\%} & \textbf{99\%} \\
     \hline
\hline
     \textbf{Global}& 1.00(0.02) & 0.60(0.07) & 0.71(0.03) & 0.98(0.03) & 1.39(0.05) & 1.69(0.14) \\
     \textbf{Local}& 0.98(0.02) & 0.05(0.04) & 0.24(0.03) & 0.84(0.03) & 2.18(0.07) & 3.90(0.23) \\
      \hline
    \end{tabular}
\end{scriptsize}}
  \caption{Mean (standard deviation) of \emph{determinant of Jacobian} of $\Phi^{-1}$ for global and local regularization with $\lambda_{\text{TV}}=0.1$ and $\lambda_{\text{OMT}}=50$ for CUMC12 within the brain. Local metric optimization (local) improves target overlap measures (see Fig.~\ref{fig:boxplot_overlap_3d_test_cumc12}) at the cost of less regular deformations than for global multi-Gaussian regularization. However, the reported determinants of Jacobian are still all positive, indicating no folding.}
  \label{tab:jacobian_across_stages_cumc12_3d}
  \vspace{-0.3cm}
 \end{table}

\vspace{-0.2cm}
\section{Conclusions}
\label{sec:discussion}

We proposed an approach to learn a \emph{local} regularizer, parameterized by a CNN, which integrates with deformable registration models and demonstrates good performance on both synthetic and real data.
While we used vSVF for computational efficiency, our approach could directly be integrated with LDDMM (resulting in local, time-varying regularization). It could also be integrated with predictive registration approaches, \eg, \cite{yang2017quicksilver}. Such an integration would remove the computational burden of optimization at runtime, yield a fast registration model, allow end-to-end training and, in particular, promises to overcome the two key issues of current deep learning approaches to deformable image registration: (1) the lack of control over spatial regularity of approaches training mostly based on image similarities and (2) the inherent limitation on registration performance by approaches which try to learn optimal registration parameters for a given registration method and a {\it chosen} regularizer. 

\vskip0.5ex
To the best of our knowledge, our model is the first approach to learn a local regularizer of a registration model by predicting local multi-Gaussian pre-weights. This is an attractive approach as it (1) allows retaining the theoretical properties of an underlying (well-understood) registration model, (2) allows imposing bounds on local regularity, and (3) focuses the effort on learning some aspects of the registration model from data, while refraining from learning the {\it entire} model which is inherently ill-posed. The estimated local regularizer might provide useful information in of itself and, in particular, indicates that a spatially non-uniform deformation model is supported by real data. 

Much experimental and theoretical work remains. More sophisticated CNN models should be explored; the method should be adapted for fast end-to-end regression; more general parameterizations of regularizers should be studied (\eg, allowing sliding), and the approach should be developed for LDDMM.

\noindent {\bf Acknowledgements.} This work was supported by grants NSF EECS-1711776, NIH 1-R01-AR072013 and the 
Austrian Science Fund (FWF project P 31799).

\clearpage
\bibliographystyle{ieee}
\bibliography{ms}

\clearpage
\setcounter{page}{1}
\appendix
\section{Supplementary material}

This supplementary material contains additional information describing our approach. \S\ref{sec:sqrt_model} discusses the theoretical properties of our model and proves that the resulting spatial transformations are diffeomorphic in the continuum. Possible undesirable effects of the numerical discretization are not studied or addressed in this work. \S\ref{sec:implementation_details} provides some critical implementation details for the CNN regressing the local pre-weights of the multi-Gaussian regularizer based on an input image. Lastly, \S\ref{sec:synthetic_experiment_setup} provides details on how the synthetic data for our synthetic experiments was created.

\subsection{Localized multi-Gaussian kernels}
\label{sec:sqrt_model}

Starting from a sum of kernels 
$\sum_{i = 0}^{N-1} w_i G_i$, we let the coefficient $w_i$ be spatially varying. In order to ensure the diffeomorphic property of deformations, we set the weights $w_i(x) = G_{\sigma_{\text{small}}} \star \omega_i(x) + \varepsilon_i$, where $\omega_i(x)$ are pre-weights which are convolved with a Gaussian filter with small standard deviation and $\varepsilon_i$ is a small positive real that acts as a constant offset parameter\footnote{We enforce this small positive constant by clamping the pre-weights to $[\epsilon,1]$. One could also directly integrate this into the weighted linear softmax definition by clamping to $[\epsilon,1]$ instead of $[0,1]$.}. We have
\begin{multline}\label{EqCriteria}
\on{Reg}_{\text{vSVF}} = \lambda \langle m_0,v_0 \rangle + \lambda_{\text{OMT}}\int \on{OMT}(w(x)) \ud x~+ \\ \lambda_{\text{TV}}\sqrt{\sum_{i=0}^{N-1}\left(\int \gamma(\| \nabla I_0(x)\|) \| \nabla \omega_i(x)\|_{2} \ud x\right)^2}\,,
\end{multline}
where $m_0$ and $v_0$ are the initial momentum and vector field, respectively. Note that the partition of unity defining the metric, intervenes in the $L^2$ scalar product $\langle m_0,v_0 \rangle$ since, with $\varepsilon_i > 0$ a positive offset,
\begin{multline}\label{EqNewSmoothing}
v_0(x) =( K(w) \star m_0)(x) \\= \sum_{i = 0}^{N-1} \sqrt{w_i(x)} \int_{y} G_i(| x - y |) \sqrt{w_i(y)} m_0(y) \on{d}\!y\,,
\end{multline}
whose spatial smoothness is enough to guarantee the deformation to be diffeomorphic.
Due to the convolution of the pre-weights, the vector field $v_0$ has a bounded norm in the space of $C^1$ vector fields which implies that its flow is a diffeomorphism at every time.
In fact, we have:
\begin{proposition}
The minimization of the objective functional \eqref{EqCriteria} over a collection of image pairs provides diffeomorphic deformations for every pair of images. At every stage of the optimization procedure, the deformations are guaranteed to be diffeomorphic.
\end{proposition}

\begin{proof}
We have the existence of a constant $K$ such that 
\begin{equation}
\| f \|_{C^1} \leq K\| f \|_{H_i} \leq K\| f \|_{H_N} \,,
\end{equation}
for every $f \in H_N$.
\par
Denote by $\Phi: (I,m) \mapsto \omega$ the nonlinear map learnt by the neural network.
At every step of the optimization, and at convergence (for a finite sample of pairs of images, each pair is denoted by the index $j$), the functional \eqref{EqCriteria} is finite, which implies that $\Phi(I_j,m_j)$ is pointwisely bounded on the domain and is in $TV$, therefore, $G_{\text{small}} \star w_i$ has a bounded $C^1$ norm, as well as $\sqrt{w_i}$ since $w_i > \varepsilon_i >0$. In addition, $E_j = \langle m_j, K(w) m_j \rangle$ is also finite and gives an upper bound for $\| G_N \star (w_i m_j) \|_{H_N}$. Thus, we have
\begin{multline}
\| \sum_{i = 0}^{N-1} \sqrt{w_i(x)} G_i(|x-y|) \sqrt{w_i(y)} \star m_j \|_{C^1} \\ \leq KN \sup_{i = 1,\ldots,N} (\| \sqrt{w_i} \|_{C^1} \| \sqrt{w_i} m_j \|_{H_N})\,.
\end{multline}
Therefore, the norm of the velocity field $v(x) = \sqrt{w_i(x)} G_i(|x-y|) \star \sqrt{w_i}  m_j$ is bounded in $C^1$ and its flow is a diffeomorphism.
\end{proof}
Also, there is a corresponding variational derivation of the spatially varying kernel with the square root which is presented next.

\subsubsection{Variational derivation}
Let us detail the variational definition of the spatially varying kernel used in Equation \eqref{EqNewSmoothing}. 
Consider 
\begin{equation} \label{InfNorm}
  \|v\|_{H}^2 = \inf \left\{ \sum_{i=0}^{N-1} \|v_i\|_{H_i}^2 \, \Big | \, \sum_{i= 0}^{N-1} \sqrt{w_i} v_i = v \right \}\,.
\end{equation}
Using Lagrange multipliers, we get critical points of the functional 
\begin{equation}
  \sum_{i=0}^{N-1} \frac{1}{2}\|v_i\|_{H_i}^2 + \langle p ,  \sum_{i= 0}^{N-1} \sqrt{w_i} v_i -v\rangle\,,
\end{equation}
therefore we get
\begin{equation}
  L_iv_i + w_i p= 0\, \quad \forall i = 0,\ldots,N-1\,,
\end{equation}
where $L_i$ is the inverse of the kernel $G_i$.
Hence, there exists $p$ such that 
$$\|v\|_{H}^2 = \sum_{i=0}^{N-1} \langle G_i \sqrt{w_i} p,\sqrt{w_i} p \rangle$$ 
for the norm. Moreover, since $v_i = G_i \sqrt{w_i} p$, we have
\begin{equation}
  v = \sum_{i= 0}^{N-1} \sqrt{w_i} G_i (\sqrt{w_i} p)\,.
\end{equation}

\section{Implementation details}
\label{sec:implementation_details}

\noindent
{\bf CNN initialization/penalty.} Directly using the CNN as described in \S\ref{subsection:cnn_regressor} does, in our experience, not lead to stable estimation results for the weights. Proper initialization and penalizing undesirable weights is therefore essential. Specifically, we use the following approaches:
\begin{itemize}
\item[1)] {\it Initialization:} We initialize all bias terms to zero and use the initialization scheme from \cite{He15a} for the convolutional weights. For the last batch normalization layer we initialize the slope to a small value ($0.025$) to avoid massive weight changes at the beginning as the registration is very sensitive to such changes.
\item[2)] {\it Weighted linear softmax input penalty:} As the  weighted linear softmax function clamps inputs, values within the clamping range will no longer produce gradients. In our experiments this was a highly problematic behavior as it appeared to lead to cases where one could not easily recover from poor locations in the input space to the weighted linear softmax\footnote{Similarly, if one uses a standard softmax function then exponential terms may result in very small gradients.}. Hence, we penalize the inputs when they are outside the $[0,1]$ range as follows:
  \begin{multline}
    \rp(z)=\\ \sum_{i=0}^{N-1}\left(w_i+z_i-\overline{z}-\text{clamp}_{\epsilon,1}(w_i+z_i-\overline{z})\right)^2.
  \end{multline}
  Here, $\text{clamp}_{\epsilon,1}$ clamps values to the interval $[\epsilon,1]$. An $\epsilon>0$ is required as the square root is not differentiable at zero. This penalty is integrated over all of space and added to the overall registration energy, \ie,
  \begin{equation}
    \RP(z(x)) = \int \rp(z(x))~\mathrm{d}x\enspace.\label{eq:pre_weight_input_range_penalty}
  \end{equation}
  We did not experiment with weightings of this term and simply added it as is. In practice this appeared to be fine (but may warrant further investigation) as the term results in zero penalty when the input values to the weighted linear softmax are not clamped and it is operating in its linear regime. 
\item[3)] {\it Weight decay:} We use a small weight decay \cite{Hanson88a} (set to 1e-5) applied to all the network weights. However, we did not extensively experiment with this parameter. Hence, its practical necessity is not clear to us at the moment. We added it to mitigate possible drift in the estimated parameters (\eg, very large weights of the convolutional filters). 
\end{itemize}

\section{Generation of synthetic data}
\label{sec:synthetic_experiment_setup}

To be able to validate with respect to a known ground truth we construct synthetic data as follows:
\begin{itemize}
\item[1)] We generate concentric circular regions with random radii and associate different multi-Gaussian weights to these regions. We associate a fixed multi-Gaussian weight to the background.
\item[2)] We randomly create vector momenta at the borders of the concentric circles. Specifically, we randomly create 10 different sectors and, within each sector, we randomly create either all positive or negative momenta orthogonal to the circle boundaries. These momenta are smoothed afterwards. 
\item[3)] Based on 2), we create a deformation.
\item[4)] We randomly create a noisy image of the same dimension as the image of the concentric circles and smooth it. We add this smoothed noise image to the concentric circle image and deform it and its associated weights given the deformation from 3). The resulting image is our synthetic source image. We also transform the image without noise.
\item[5)] We repeat steps 2) to 4), starting from the synthetic source image without noise. The resulting deformation is applied to the (noisy) synthetic source image to create the synthetic target image.
\end{itemize}
These steps are repeated to obtain a desired set of image pairs.
\end{document}